\documentclass[10pt,twocolumn,twoside]{IEEEtran} 

\usepackage{amsmath,amssymb}
\usepackage[boxed,linesnumbered]{algorithm2e}
\usepackage{graphicx,epsfig}
\usepackage[caption=false,font=footnotesize]{subfig}

\def\b{\mathbf{b}}
\def\d{\mathbf{d}}
\def\A{\mathbf{A}}
\def\ii{\iota}
\def\i{i}
\def\j{j}
\def\Z{\mathbb{Z}}

\ifCLASSINFOpdf
\else
\fi

\newtheorem{theorem}{Theorem}[section]
\newtheorem{proposition}[theorem]{Proposition}

\begin{document}

\title{On Fast Bilateral Filtering using  Fourier Kernels}
\author{Sanjay Ghosh, \textit{Student Member, IEEE}, and Kunal N. Chaudhury, \textit{Senior Member, IEEE}
\thanks{Address: Department of Electrical Engineering, Indian Institute of Science, Bangalore, India. Correspondence: kunal@ee.iisc.ernet.in.}}

\maketitle

\begin{abstract}
It was demonstrated in earlier work that, by approximating its range kernel using shiftable functions, the non-linear bilateral filter can be computed using a series of fast convolutions. 
Previous approaches based on shiftable approximation have, however, been restricted to Gaussian range kernels. In this work, we propose a novel approximation that can be applied to \textit{any} range kernel, provided it has a pointwise-convergent Fourier series. More specifically, we propose to approximate the Gaussian range kernel of the bilateral filter using a Fourier basis, where the coefficients of the basis are obtained by solving a series of least-squares problems. The coefficients can be efficiently computed using a recursive form of the QR decomposition. By controlling the cardinality of the Fourier basis, we can obtain a good tradeoff between the run-time and the filtering accuracy.  In particular, we are able to guarantee sub-pixel accuracy for the overall filtering, which is not provided by most existing methods for fast bilateral filtering. We present simulation results to demonstrate the speed and accuracy of the proposed algorithm.
\end{abstract}


\IEEEpeerreviewmaketitle

\section{Introduction}

The bilateral filter was introduced by Tomasi and Manduchi in \cite{Tomasi1998} as a non-linear extension of the classical Gaussian filter.
The bilateral filter employs a range kernel along with a spatial kernel for performing edge-preserving smoothing of images. 
Since its introduction, the bilateral filter has found widespread applications in image processing, computer graphics, computer vision, and 
computational photography \cite{Bennett2005} - \cite{optflow}. 

In this paper, we will consider a general form of the bilateral filter 
where an arbitrary kernel is used for the range filtering, and a box or Gaussian kernel is used for the spatial filtering \cite{Tomasi1998}. 
In particular, consider an image $f:I \rightarrow \mathbb{R}$, where $I \subset \Z^2$ is a finite rectangular lattice. The output  of the bilateral filter $ f_{\mathrm{BF}} : I \rightarrow \mathbb{R}$ is given by
\begin{equation}
\label{BF}
 f_{\mathrm{BF}}(\i)=  \frac{\sum_{\j \in \Omega} w(\j) \  \phi\big(f(\i-\j)-f(\i) \big) \ f(\i-\j)}{\sum_{\j \in \Omega} w(\j)  \  \phi\big(f(\i-\j)-f(\i) \big)  },
\end{equation}
where $\phi(t)$ is the range kernel and $w(\i)$ is the spatial kernel. The spatial kernel is usually a Gaussian \cite{Tomasi1998},
\begin{equation}
\label{spatial_kernel1}
w(\i) = \exp\left(- \lVert \i \rVert^2/2\sigma_s^2\right).
\end{equation}
The window $\Omega$ of the spatial kernel  is a local neighbourhood of the origin. For example, $\Omega=[-W,W]^2$ for the Gaussian kernel, where $W=3\sigma_s$. 
The original proposal in \cite{Tomasi1998} was to use a Gaussian range kernel given by
\begin{equation}
\label{range_kernel}
\phi(t) = \exp\left(- t^2/2\sigma_r^2\right).
\end{equation}
In more recent work, exponential range kernels have been used \cite{Gunturk2011,Mirbach2012,Yang2014}. 

The direct computation of \eqref{BF} requires $O(W^2)$ operations per pixel. In fact, the direct computation is slow for practical settings of $W$ \cite{Durand2002}.  To address this issue, researchers have come up with various fast  algorithms \cite{Durand2002} - \cite{Chaudhury2015}. While some of these algorithms can reduce the complexity to $O(1)$ operations per pixel for any arbitrary $W$, there is, however, no available guarantee on the approximation quality that can be achieved using these algorithms. In fact, as reported in \cite{Chaudhury2011}, a poor approximation can lead to visible distortions in the filtered image. Only recently, a quantitative analysis of Yang's fast algorithm was presented in \cite{An2015}.

In Section \ref{sec:SBF}, we recall the idea of constant-time bilateral filtering using Fourier (complex exponential) kernels  \cite{Chaudhury2011a}.
In this work, we build on this idea to propose a new algorithm for approximating \eqref{BF} using the shiftable Fourier basis. The contribution of this work is not the fast algorithm itself, but rather the approximation scheme  in Section \ref{sec:PFA}, and the subsequent approximation guarantee in Section \ref{sec:FA}. The approximation scheme can be applied to \textit{any} arbitrary range kernel that has a  pointwise-convergent Fourier series. In this respect, we note that all previous approaches based on shiftable approximation were restricted to Gaussian range kernels \cite{Chaudhury2011,Chaudhury2011a,Kamata2015}. We provide some representative results concerning the speed and accuracy of the resulting algorithm in Section \ref{sec:Sim}, where we also compare the empirical accuracy of the filtering with the bounds predicted by our analysis.

\section{Shiftable Bilateral Filtering}
\label{sec:SBF}

It was demonstrated in \cite{Chaudhury2011,Chaudhury2011a} that the bilateral filter can be decomposed into a series of Gaussian convolutions using shiftable functions. 
In particular, since our present interest is in the shiftable complex exponential, consider the function 
\begin{equation}
\label{compExp}
\varphi_N(t) = \sum_{n=-N}^N \!\! c_n \exp(\ii n \omega t),
\end{equation}
where $\ii^2=-1$. By setting \eqref{compExp} as the range kernel $\phi(t)$, we can decompose the numerator in \eqref{BF} as
\begin{equation*}
\sum_{n=-N}^N \!\! c_n \exp\! \big(\!-\ii n w f(\i) \big) F_n(\i),
\end{equation*}
where 
\begin{equation}
\label{gauss}
F_n(\i)= \sum_{\j \in \Omega} w(\j) f(\i-\j) \exp\big(\ii n \omega f(\i-\j)\big).
\end{equation}
It is clear that a similar decomposition can be obtained for the denominator of \eqref{BF}.
We readily recognize \eqref{gauss} to be a Gaussian convolution. As is well-known, the Gaussian convolution in \eqref{gauss}  can be efficiently implemented at constant-time complexity (with respect to $\sigma_s$) using separability and recursion \cite{Deriche1993}.  In summary, we can decompose the bilateral filtering into a series of Gaussian filtering. The fast \textit{shiftable} algorithm resulting from this decomposition is summarized in Algorithm \ref{algo1}. We use $G^{\ast}(\i)$  in line \ref{conj} to denote the complex-conjugate of $G(\i)$. In line \ref{conv}, we use $\bar{F}$ and $\bar{G}$ to denote the Gaussian filtering of the images $F(\i)$ and $G(\i)$. To avoid confusion, we note that the formal structure of Algorithm \ref{algo1} is somewhat different from that of the shiftable algorithms in \cite{Chaudhury2011,Chaudhury2011a}. While the cosine and sine components of the complex exponential were used in \cite{Chaudhury2011,Chaudhury2011a}, we work directly with the complex exponential in Algorithm \ref{algo1}. Note that we have abused notation in using $f_{\text{BF}}(\i)$ to denote the shiftable approximation of \eqref{BF} in Algorithm \ref{algo1}.

\IncMargin{1.5mm}
\begin{algorithm}
\KwData{Image $f: I \rightarrow \mathbb{R}$\; 
\textbf{Parameters}:  Filter $w(\i)$, and $\omega$, $N$, and $(c_n)_{-N \leq n \leq N}$.}
\KwResult{Shiftable approximation of \eqref{BF}.}
\textbf{Initialize}: Set $P(\i)=0$ and $Q(\i)=0$ for all $\i \in I$\;
 \For{$n=-N,\ldots,N$}{
$G(\i) =\exp\left(\imath n \omega f(\i)\right)$ for $\i \in I$\;
$F(\i) =G(\i) f(\i)$ for $\i \in I$\; 
$H(\i)=c_n G^{\ast}(\i)$ for $\i \in I$\;   \label{conj}
Compute $\bar{G} = F \ast w$ and $\bar{G}  = G \ast w$\; \label{conv}
$P(\i)=P(\i) + H(\i) \cdot \bar{F} (\i)$ for $\i \in I$\; 
$Q(\i)=Q(\i)+ H(\i) \cdot  \bar{G} (\i)$ for $\i \in I$\; 
}
Set $f_{\text{BF}}(\i)=P(\i)/Q(\i)$ for all $\i \in I$.
\caption{Shiftable Bilateral Filtering.}
\label{algo1}
\end{algorithm}
\DecMargin{0.5mm}

If the range kernel is not shiftable, one can approximate it using a shiftable function.
For example, the shiftable raised-cosines were used in \cite{Chaudhury2011} to approximate the Gaussian kernel.
Shiftable approximation using polynomials was later presented in \cite{Chaudhury2011a}. More recently, the classical Fourier basis was used for this purpose in \cite{Kamata2015}. The above approximations, however, come with the following shortcomings: \newline
$\bullet$ They are customized to work with the Gaussian kernel, and cannot be extended to general range kernels, such as the exponential kernel \cite{Mirbach2012,Yang2014}. Even for the Gaussian kernel, the proposal in \cite{Kamata2015} requires one to compute the  coefficients of the Fourier series. This is computationally intensive (e.g., requires numerical integration, or some analytical properties particular to the kernel), and cannot be done on-the-fly. Indeed, the authors in \cite{Kamata2015} work with an approximation of the Fourier coefficients, which is only valid for small $\sigma_r$. \newline
$\bullet$  Notice that, in most applications of the bilateral filter, the argument $t$ in \eqref{range_kernel} assumes discrete values. This should be taken into consideration while designing the shiftable approximation. The approximations in \cite{Chaudhury2011,Kamata2015}, however, do not necessarily guarantee that the approximation error at these discrete points are within some user-defined tolerance. This makes it difficult to quantify the overall filtering accuracy. 
In this paper, we propose a rather simple optimization principle, which has an efficient implementation. This provides us with the desired control on the numerical accuracy of the overall filtering.

\section{Progressive Fourier Approximation}
\label{sec:PFA}

We now explain how the above shortcomings can be fixed. 
As noted above, the argument $t$ in \eqref{range_kernel} takes on the values $|f(\i -\j) - f(\j)|$ as $\i$ and $\j$ varies over the image. In particular, $t$ takes values in $\Lambda_T = \{0,1,\ldots,T\}$, where
\begin{equation*}
T = \max_{\i \in I} \ \max_{\j \in \Omega} \ |f(\i-\j) - f(\j)|. 
\end{equation*}
Thus, $T$ is the dynamic range of the image measured over the window $\Omega$, which is typically smaller than the full dynamic range. We can compute $T$ using the fast algorithm in \cite{Chaudhury2013}; the run-time of the algorithm does not depend on the size of $\Omega$. Without loss of generality, we assume that the range kernel  $\phi(t)$ is symmetric. The problem is that of approximating $\phi(t)$ using a shiftable function over the half-interval $[0,T]$.  We propose to use the shiftable Fourier basis for this purpose. In particular, we fix some order $N \geq 1$, and consider the shiftable function
\begin{equation}
\label{FourierExp}
\varphi_N(t) = d_0 + \sum_{n=1}^N d_n \cos \left(n\omega t \right),
\end{equation}
where $\omega = \pi/T$.  As is well-known, using the identity $\cos\theta = (\exp(\ii \theta) + \exp(-\ii \theta))/2$, we can write \eqref{FourierExp} as in \eqref{compExp}, where $c_0=d_0$, and $c_{\pm n}=(1/2)d_n$ for $n=1,\ldots,N$. The key difference with \cite{Kamata2015} is with respect to the rule used to set the coefficients $d_0,\ldots,d_N$ in \eqref{FourierExp}. These are set to be the standard Fourier coefficients of $\phi(t)$ in \cite{Kamata2015}. 
In keeping with the arguments presented in earlier, we take a different approach and instead try to minimize the error $\phi(t) - \varphi_N(t)$ at the discrete points $t \in \Lambda_T$. In particular, we consider the problem of finding $d_0,\ldots,d_N$ that minimizes the gross error
\begin{equation}
\label{grossError}
\sum_{t \in \Lambda_T} \ \big(\phi(t) - \varphi_N(t)\big)^2.
\end{equation}
This is the classical linear least-squares problem, where the unknowns are $d_0,\ldots,d_N$. Indeed, using matrix-notation, we can write \eqref{grossError} as $\lVert \b - \A \d \rVert^2$, where $\d=(d_0,\ldots,d_N)$, $\b$ is the discretization of $\phi(t)$ at the points $t \in \Lambda_T$, and the columns of $\A$ are the corresponding discretization of the basis functions in \eqref{FourierExp}. In particular, let us denote 
\begin{equation}
\label{lsqlin}
\mathcal{E}_N = \min_{\d \in \mathbb{R}^{N+1}} \ \lVert \b - \A \d \rVert^2.
\end{equation}
The following fact is the basis of our approximation algorithm to be discussed next.
\begin{proposition}[Decay of Error]
\label{prop}
Assume that the Fourier series of the range kernel converges pointwise on the interval $[-T,T]$. That is, for $t \in [-T,T]$,
\begin{equation*}
\lim_{N \rightarrow \infty}  \varphi_N(t)  = \phi(t),
\end{equation*}
where $d_0,\ldots,d_N$ in \eqref{FourierExp} are the Fourier coefficients of $\phi(t)$. Then $\mathcal{E}_N$ decays to zero as $N \rightarrow \infty$.
\end{proposition}
\begin{proof}
Indeed, let $e_N$ be the error in \eqref{grossError} when $\varphi_N(t)$ is taken to be the $N$-th order Fourier approximation of $\phi(t)$. Then, by optimality, we have $\mathcal{E}_N \leq e_N$. Since, by assumption, $e_N \rightarrow 0$ as $N \rightarrow \infty$, the proposition follows. 
\end{proof}
We note that the Fourier series converges pointwise for any continuously-differentiable function, e.g., Gaussian and polynomials. Convergence is also guaranteed for functions that are continuous and piecewise-differentiable \cite{Grafakos}, such as the exponential. Thus, the assumption in Proposition \ref{prop} covers the commonly used kernels \cite{Tomasi1998,Mirbach2012,Yang2014}.

Proposition \ref{prop} suggests the following numerical scheme: We fix some user-defined tolerance $\varepsilon^2$. We begin with $N=1$, and solve \eqref{lsqlin} to get $\mathcal{E}_N$. If $\mathcal{E}_N < \varepsilon^2$, we stop. Else, we increase $N$ by one and proceed, until $\mathcal{E}_N \leq \varepsilon^2$. In other words, we solve a series of least-squares problems, where the basis matrix $\A$ at each step is obtained by augmenting the $\A$ in the previous step. The whole process can be efficiently implemented  using a recursive version of the modified QR algorithm \cite{Dummel}. The main idea is that \eqref{lsqlin} can be computed by solving $\mathbf{R} \d = \mathbf{Q}^T \b$ using back-substitution, where $\A=\mathbf{Q} \mathbf{R}$ is the QR-decomposition of $\A$.
In the recursive computation, $\mathbf{Q},\mathbf{R}$, and $ \mathbf{Q}^T \b$ at each iteration is computed from the corresponding quantities in  the previous iteration using cheap operations.
An adaptation of this recursive algorithm to our problem is provided in Algorithm \ref{algo2}. 
In steps \ref{sample1} and \ref{sample2}, we discretize the kernel and the incoming Fourier basis. 
In step \ref{rk}, $r_k$ denotes the $k$-th component of $\mathbf{r}$.

\section{Filtering Accuracy}
\label{sec:FA}

Suppose we are given a range kernel $\phi(t)$ and tolerance $\varepsilon$. We compute the approximation order $N$ and the corresponding coefficients $d_0,\ldots,d_N$ using Algorithm \ref{algo2}. This gives us the corresponding kernel $\varphi_N(t)$ in \eqref{compExp}, which is used to approximate \eqref{BF} using Algorithm \ref{algo1}. In particular, the approximation provided by Algorithm \ref{algo1} is given by
\begin{equation}
\label{SBF}
\hat{f}_{\mathrm{BF}}(\i)=  \frac{\sum_{\j \in \Omega} w(\j)   \varphi_N \big(f(\i-\j)-f(\i) \big)  f(\i-\j)}{\sum_{\j \in \Omega} w(\j)   \varphi_N \big(f(\i-\j)-f(\i) \big)  }.
\end{equation}
By construction, for all $t \in \Lambda_T$,
\begin{equation}
\label{eps}
\lvert \phi(t)- \varphi_N(t)  \rvert \leq \varepsilon.
\end{equation}
Similar to \cite{An2015}, we consider the $\ell_{\infty}$ (worst-case) error
\begin{equation}
\label{err_infinity}
\lVert  f_{\mathrm{BF}} -  \hat{f}_{\mathrm{BF}} \rVert_{\infty} = \max \big\{ |f_{\mathrm{BF}}(\i) -   \hat{f}_{\mathrm{BF}}(\i)| :  \ \i \in I \big\}.
\end{equation}
Our goal is to bound \eqref{err_infinity}, which provides us with an estimate of the pixelwise difference between the outputs of the exact and the approximate bilateral filter. In fact, a simple analysis (cf. Appendix) give us the following result.
\begin{proposition}[Filtering Accuracy]
\begin{equation}
\label{accuracy}
\lVert  f_{\mathrm{BF}} - \hat{f}_{\mathrm{BF}} \rVert_{\infty}  \leq \frac{2T \varepsilon  }{w(0) -\varepsilon }.
\end{equation}
\end{proposition}
In other words, the filtering error is essentially within a certain factor of the kernel approximation error $\varepsilon$.
To arrive at  \eqref{accuracy}, we have assumed that the weights of the spatial filter add up to unity.
Indeed, this assumption can be made since the spatial filter appears in both  the numerator and denominator of \eqref{BF} and \eqref{SBF}.

\IncMargin{1.5mm}
\begin{algorithm}[!htp]
\KwData{Kernel $\phi(t)$, half-period $T$, and tolerance $\varepsilon$.}
\KwResult{$N$ and $\d \in \mathbb{R}^N$.}
$\omega = \pi/T$\;
$\Lambda_T = \{0,1,\ldots,T\}$\;
$\b = [\phi(t)]_{t \in \Lambda_T} \in \mathbb{R}^{T+1}$\; \label{sample1}
\textbf{Initialize}: 
$N=1$\;
Set $\mathbf{a} \in \mathbb{R}^{T+1}$ to be the all-ones vector\;
$\A = \mathbf{a}$\;
 $\mathbf{R} = \lVert \mathbf{a} \rVert$\;
$\mathbf{Q} = \mathbf{a}/\mathbf{R}$\;
$\mathbf{p} = \mathbf{Q}^T \b$\;
$\mathcal{E}= \lVert \b -  \mathbf{Q}  \mathbf{p}  \rVert$\;
\While{$\mathcal{E} > \varepsilon$}{
Set $\mathbf{r} \in \mathbb{R}^{N}$ to be the all-zeros vector\;
$N=N+1$\;
$\mathbf{a} = [\cos \left(N \omega t \right)]_{t \in \Lambda_T} \in \mathbb{R}^{T+1}$\; \label{sample2}
$\A = [\A \ \lvert \ \mathbf{a}]$\;
\For{$k=1,\ldots,N-1$}{
Set $\mathbf{q}_k$ to be the $k$-th column of $\mathbf{Q}$\;
$r_k =  \mathbf{a}^T\mathbf{q}_k$\;   \label{rk}
$\mathbf{a} = \mathbf{a}  - \mathbf{r}_k \mathbf{q}_k$\;
}
$r_N = \lVert \mathbf{a} \rVert$\;
$\mathbf{a}= (1/r_N)\mathbf{a}$\;
$\mathbf{Q} = [\mathbf{Q} \ \lvert \  \mathbf{a}] \in \mathbb{R}^{(T+1) \times N}$\;
$\mathbf{p} = [\mathbf{p} \ \lvert \ \mathbf{a}^T \b] \in \mathbb{R}^{N}$\;
Add a row of zeros to $\mathbf{R}$\;
$\mathbf{R} = [\mathbf{R} \ \lvert \ \mathbf{r}]$\; 
Solve $\mathbf{R}\d = \mathbf{p}$ using back-substitution\;
$\mathcal{E}= \lVert \b - \A \d\rVert$\;
}
\caption{Progressive Fourier Approximation.}
\label{algo2}
\end{algorithm}
\DecMargin{1.5mm}

\section{Simulation and Conclusion}
\label{sec:Sim}

All simulations reported here were performed using Matlab 8.4 on a MacBook Air with 1.3 GHz Intel Core i5 processor and 4 GB memory. The typical run-time of Algorithm \ref{algo2} was between $1$-$15$ milliseconds (depending on the order $N$) for the simulations reported in this section. This is a small fraction of the overall run-time of Algorithm \ref{algo1}. Indeed, the time required to filter a single $512 \times 512$ image with a Gaussian kernel is already about $20$ milliseconds. In Figure \ref{approx1}, we give an example of the approximation result obtained using Algorithm \ref{algo2} with $\varepsilon=1\text{e-}3$. In Figure \ref{approx2}, we compare the coefficients obtained using Algorithm \ref{algo2} with that obtained by expanding the raised-cosines \cite{Chaudhury2011} into \eqref{compExp}. Notice that the former decays much more rapidly and hence requires fewer terms.

\begin{figure}[!htp]
  \centering
   \subfloat[]{\includegraphics[width=0.48 \linewidth]{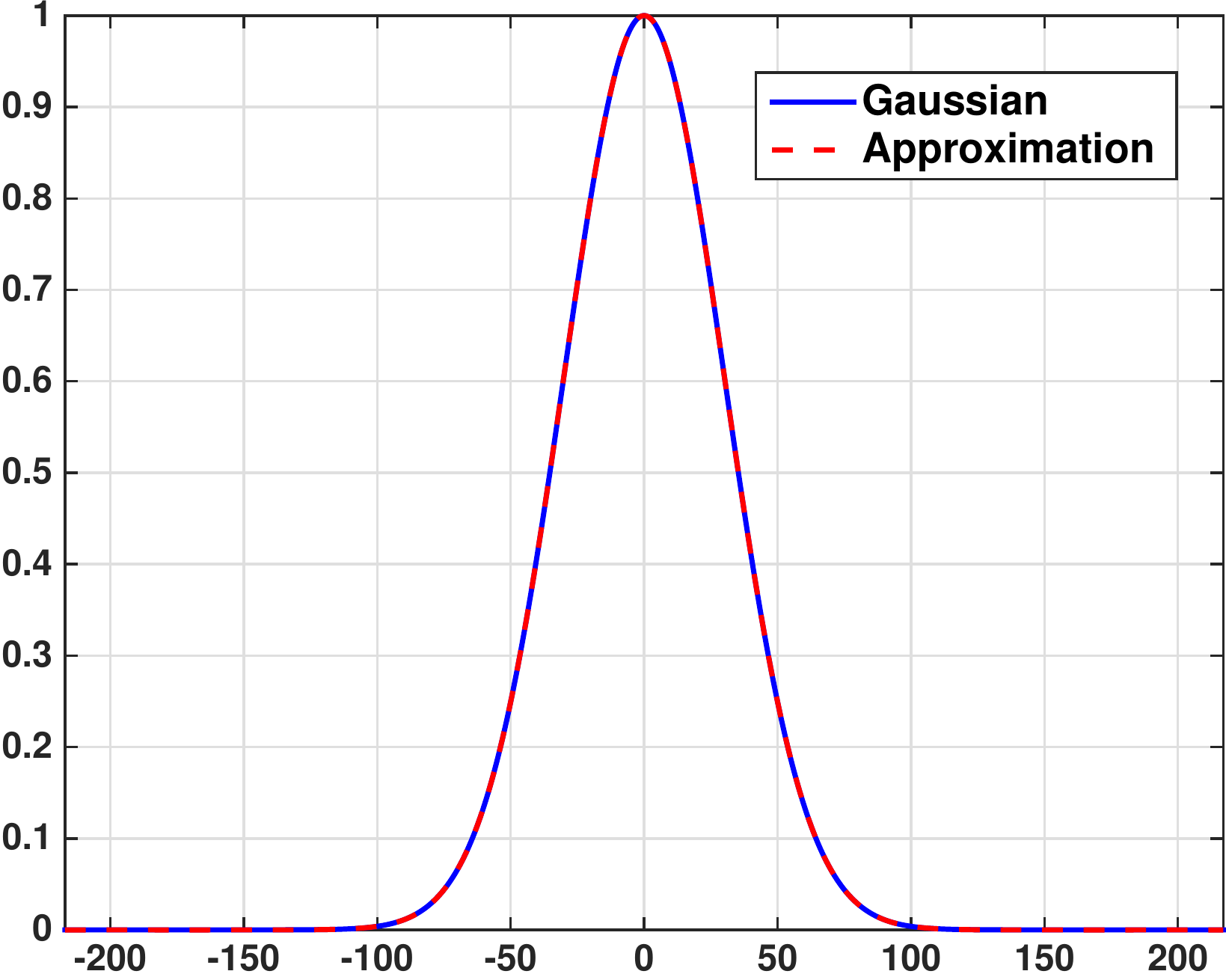}} \hspace{-0.5mm}
  \subfloat[]{\includegraphics[width= 0.492 \linewidth]{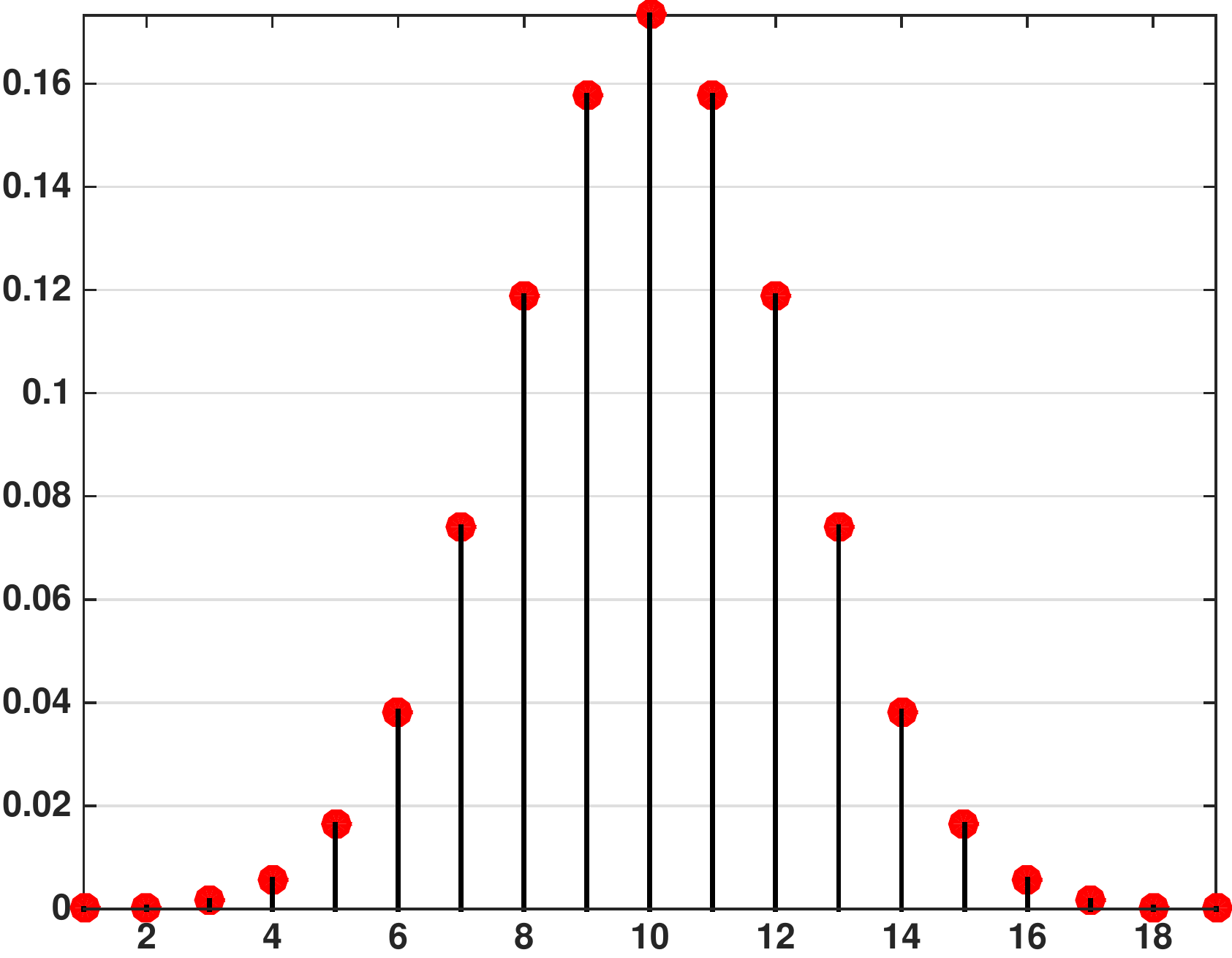}} 
  \caption{Left: Target Gaussian $(\sigma_r = 30)$ and the Fourier approximation ($N=9$) obtained using Algorithm \ref{algo1}. Right: Coefficients $c_{-9},\ldots,c_9$.}
    \label{approx1}
\end{figure} 

\begin{figure}[!htp]
  \centering
\includegraphics[width=0.7 \linewidth]{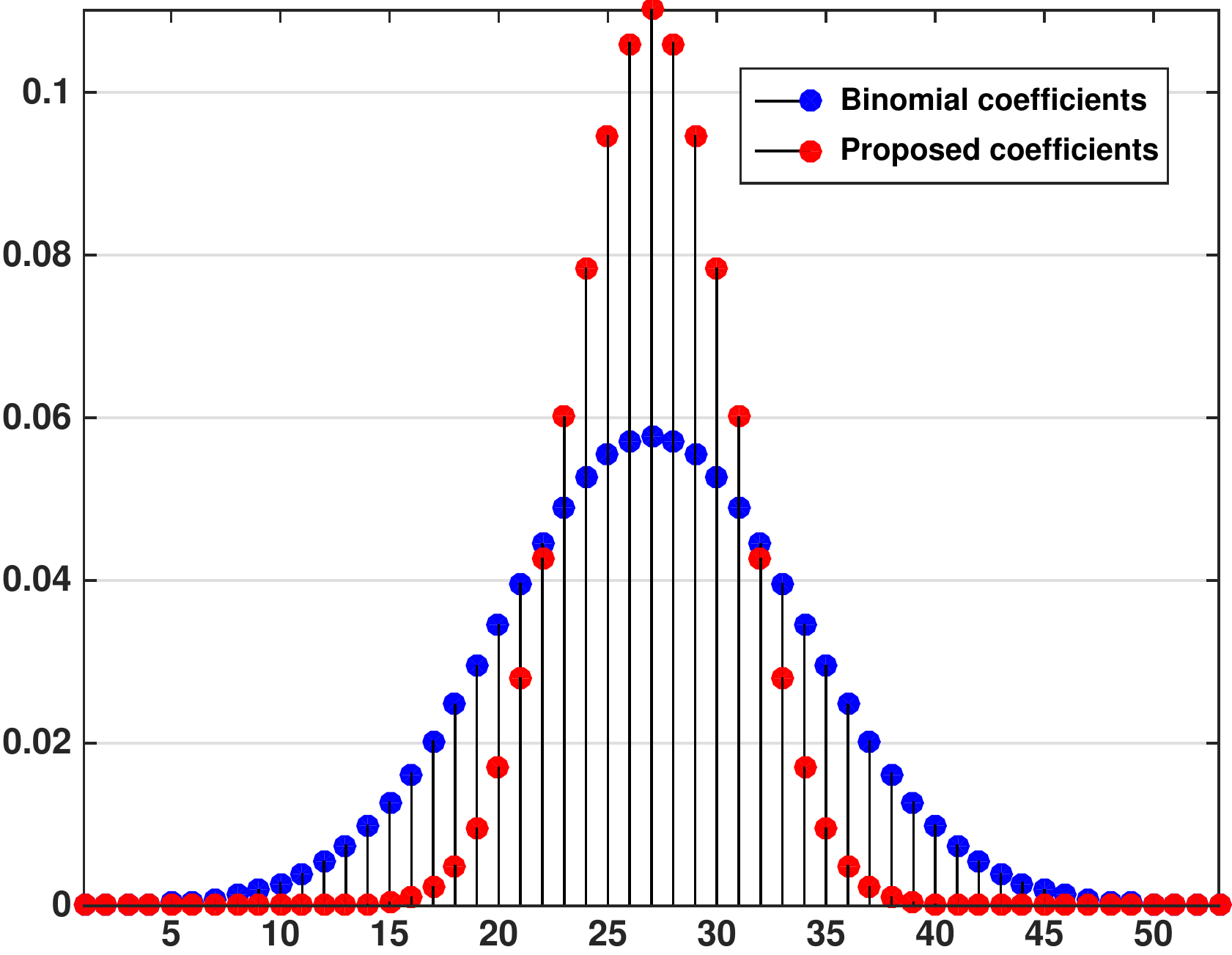}
  \caption{Comparison of the normalized binomial coefficients from \cite{Chaudhury2011} with that obtained using Algorithm \ref{algo2} ($\varepsilon=1\text{e-}3$) when $\sigma_r = 10$.}
    \label{approx2}
\end{figure} 

We present some results on the \textit{Barbara} image for which $T$ was computed to be $217$. We note that the run-time of the direct implementation of \eqref{BF} depends only on the image size and $\sigma_s$. On the other hand, the run-time of the proposed algorithm depends on $T$, tolerance $\varepsilon$, image size, and $\sigma_r$. The fact that the run-time is almost independent of $\sigma_s$ (constant-time algorithm) is evident from the results in Table \ref{table1}. The small fluctuations are essentially due to the variable padding required to handle the boundary conditions for the spatial filtering. 
 
\begin{table}[!htp]
\centering
\caption{Run-time for the $512 \times 512$ \textit{Barbara} image at different $\sigma_s$ and fixed $\sigma_r=30$. The run-time of the direct implementation is $95$ seconds.}
\label{table1}
\begin{tabular}{|c |p{0.6cm}|p{0.6cm}|p{0.6cm}|p{0.6cm}|p{0.6cm}|p{0.6cm}|p{0.6cm}|}
 \hline
$\sigma_s$   &1      &2  &5 &8 &10 &12   \\ \hline
Fast ($\varepsilon=1\mbox{e-}3$) &630ms  &635ms &638ms &640ms &645ms &650ms  \\ \hline
\end{tabular}
\end{table} 

The run-time of the proposed algorithm scales inversely with $\sigma_r$, which was also observed for the shiftable filtering in \cite{Chaudhury2011,Chaudhury2013,Kamata2015}. In particular, as $\sigma_r$ gets small, the Gaussian  range kernel tends to a Dirac-like distribution \cite{Grafakos}. As is well-known, the Dirac distribution is formally composed of all frequencies. The implication of this fact is that a large $N$ is required to approximate the kernel for small $\sigma_r$, and hence the increase in run-time. This is demonstrated with an example in Table \ref{table2}. However, notice that even for small $\sigma_r$, the proposed algorithm is much faster than the exact implementation. For $\sigma_r > 20$, the speedup is by a couple of orders.

\begin{table}[!htp]
\centering
\caption{Same as in Table \ref{table1}, except that $\sigma_r$ is varied and $\sigma_s=3$. The run-time of the direct implementation is $95$ seconds.}
\label{table2}
\begin{tabular}{|c |p{0.6cm}|p{0.6cm}|p{0.6cm}|p{0.6cm}|p{0.6cm}|p{0.6cm}|p{0.6cm}|}
 \hline
$\sigma_r$   &10      &15  &20 &30  &50 &100   \\ \hline
Fast ($\varepsilon=1\mbox{e-}3$) &2.1s  &1.5s &840ms &634ms &450ms &200ms  \\ \hline
\end{tabular}
\end{table} 

In Table \ref{table3}, we present the variation of run-time with tolerance $\varepsilon$ for a fixed filter setting. It is seen that the order $N$ and hence the run-time changes rather slowly with $\varepsilon$ (almost logarithmically in $1/\varepsilon$). We have however not been able to establish this empirical fact, which is deeply tied to the working of Algorithm \ref{algo2}.
We next compare the bound in \eqref{accuracy} with the actual $\ell_{\infty}$ error for the \textit{Barbara} image in Table \ref{table4}. We note that the error is within the predicted bound. 
In fact, we are able to predict sub-pixel accuracy when $\varepsilon<1\text{e-}5$. The bounds are, however, far from being tight. One of the reasons for this is that we have not incorporated any information about the local intensity distribution into our analysis. Derivation of a tighter bound will require a more sophisticated analysis. The present work is a first step in that direction. To best of our knowledge, with the exception of \cite{Yang2009}, this is the only approximation algorithm that comes with a provable guarantee on the filtering accuracy.

\begin{table}[!htp]
\centering
\caption{Variation of the run-time with $\varepsilon$ for the \textit{Barbara} image when $\sigma_s=3$ and $\sigma_r=30$. Also shown is the order $N$.}
\label{table3}
\begin{tabular}{|c |p{0.6cm}|p{0.6cm}|p{0.6cm}|p{0.6cm}|p{0.6cm}|p{0.6cm}|}
 \hline
$\varepsilon$   &1e-5  &1e-4 &1e-3  &0.01 &0.1   \\ \hline
$N$   &12 &11  &10 &8  &7    \\ \hline
Fast ($\varepsilon=1\mbox{e-}3$) &910ms &850ms &780ms &715ms &670ms  \\ \hline
\end{tabular}
\end{table}

%\begin{figure}
%  \centering
%   \subfloat[Input.]{\includegraphics[width=0.42 \linewidth]{figures/input.eps}} \hspace{-0.5mm}
%  \subfloat[Direct Implementation.]{\includegraphics[width= 0.42 \linewidth]{figures/direct.eps}} \\
%  \subfloat[Proposed Algorithm.]{\includegraphics[width= 0.42 \linewidth]{figures/fast.eps}}  \hspace{-0.5mm}
%  \subfloat[Difference Image.]{\includegraphics[width= 0.474 \linewidth]{figures/diff.eps}}
%  \caption{Outputs of the exact and fast implementations at $\sigma_s=3$ and $\sigma_r = 30$.}
%    \label{visual}
%\end{figure} 

%To conclude, we present a visual comparison of the filtering results for the \textit{Barbara} image in Figure \ref{visual}. We notice that the outputs of the direct and the fast implementations are almost indistinguishable. Indeed, the mean-squared-error between (b) and (c) is $4\text{e-}8$, and the maximum pixelwise difference is $0.01$.

\begin{table}[!htp]
\centering
\caption{Comparison of the predicted bound and the actual $\ell_{\infty}$ error for the \textit{Barbara} image at $\sigma_s=3$ and $\sigma_r=30$.}
\label{table4}
\begin{tabular}{|c |p{0.8cm}|p{0.8cm}|p{0.8cm}|p{0.8cm}|p{0.8cm}|p{0.8cm}|}
 \hline
$\varepsilon$   &1e-8 &1e-5  &1e-4 &1e-3  &0.01    \\ \hline
$N$   &15 &12  &11 &10  &8    \\ \hline
Actual $\eqref{err_infinity}$ &2.7e-8 &1.1-4 &9e-4 &0.01 &0.3   \\ \hline
Predicted $\eqref{accuracy}$ &2.4e-4 &0.2 &2.5 &29.5 &561  \\ \hline
\end{tabular}
\end{table} 

\section{Acknowledgement}

This work was supported by the Startup Grant awarded by the Indian Institute of Science. The authors would like to thank the anonymous reviewers for their comments and suggestions.

\section{Appendix}

In this section, we outline the main steps in the derivation of \eqref{accuracy}. We write \eqref{BF} as $f_{\mathrm{BF}}(\i)=  P_1(\i)/Q_1(\i)$, where 
\begin{equation*}
P_1(\i)=\sum_{\j \in \Omega} w(\j)  \phi \big(f(\i-\j)-f(\i) \big) f(\i-\j),
\end{equation*}
and
\begin{equation*}
Q_1(\i)= \sum_{\j \in \Omega} w(\j) \phi \big(f(\i-\j)-f(\i) \big).
\end{equation*}
Similarly, we write \eqref{SBF} as $\hat{f}_{\mathrm{BF}}(\i) = P_2(\i)/Q_2(\i)$, where 
\begin{equation*}
P_2(\i)=\sum_{\j \in \Omega} w(\j) \varphi_N \big(f(\i-\j)-f(\i) \big)  f(\i-\j),
\end{equation*}
and
\begin{equation*}
Q_2(\i)= \sum_{\j \in \Omega} w(\j) \varphi_N \big(f(\i-\j)-f(\i) \big) .
\end{equation*}
Then $f_{\mathrm{BF}} (\i) -  \hat{f}_{\mathrm{BF}}  (\i)$ can be expressed as
\begin{equation}
\label{diff}
\frac{1}{Q_2(\i)} \Big[f_{\mathrm{BF}}(\i) \big(Q_2(\i)-Q_1(\i)\big) + P_1(\i)-P_2(\i) \Big].
\end{equation}
From  \eqref{eps}, we have  $\lVert Q_1 - Q_2 \rVert_{\infty}\leq \varepsilon$.
On the other hand, note that $\lVert  f_{\mathrm{BF}}  \rVert_{\infty}  \leq T$. 
This is because $f_{\mathrm{BF}} (\i)$ is given by the convex combination of $\{f(\i- \j) : \j \in \Omega\}$. 
Therefore, from \eqref{eps}, we get $\lVert P_1 -P_2 \rVert_{\infty}\leq T \varepsilon$.

To obtain a lower-bound for $Q_2(\i)$ in \eqref{diff}, we note that
\begin{equation*}
Q_1(\i) = w(0) \varphi(0) + \text{positive terms} \geq w(0),
\end{equation*}
where we have used the non-negativity of the range and spatial kernels. Therefore, using the inverse triangle inequality, we get
\begin{equation*}
\label{lb}
|Q_2(\i)| \geq Q_1(\i) -  |Q_2(\i) - Q_1(\i)| \geq   w(0) - \varepsilon.
\end{equation*}
By incorporating the above bounds into \eqref{diff}, we get \eqref{accuracy}.

\end{document}